\documentclass[11pt,a4paper]{article}

\usepackage[utf8]{inputenc}
\usepackage[T1]{fontenc}
\usepackage{lmodern}
\usepackage{amsmath,amssymb,amsthm}
\usepackage{mathtools}
\usepackage{bm}
\usepackage{booktabs}
\usepackage{graphicx}
\usepackage[margin=1in]{geometry}
\usepackage[hypertexnames=false]{hyperref}
\usepackage{cleveref}
\usepackage{algorithm}
\usepackage{algpseudocode}
\usepackage{natbib}
\usepackage{microtype}
\usepackage{xcolor}
\usepackage{multirow}

\newtheorem{theorem}{Theorem}

\newtheorem{proposition}[theorem]{Proposition}
\newtheorem{definition}[theorem]{Definition}
\newtheorem{remark}[theorem]{Remark}
\newtheorem{assumption}[theorem]{Assumption}

\newcommand{\R}{\mathbb{R}}
\newcommand{\dd}{\mathrm{d}}
\newcommand{\pd}[2]{\frac{\partial #1}{\partial #2}}
\newcommand{\thetastar}{\bm{\theta}^*}
\newcommand{\thetabeta}{\bm{\theta}^\beta}
\newcommand{\bomega}{\bm{\omega}}
\newcommand{\btheta}{\bm{\theta}}
\newcommand{\transpose}{\mathsf{T}}

\begin{document}

\title{%
  The Phase Is the Gradient: Equilibrium Propagation\\
  for Frequency Learning in Kuramoto Networks
}

\author{
  Mani Rash Ahmadi\\
  Calibur Labs\\
  \texttt{mani@caliburlabs.com}
}

\date{April 2026}

\maketitle

\begin{abstract}
We prove that in a coupled Kuramoto oscillator network at stable
equilibrium, the physical phase displacement under weak output nudging
is the gradient of the loss with respect to natural frequencies, with
equality as the nudging strength~$\beta \to 0$. Prior oscillator
equilibrium propagation work explicitly set aside natural frequency as a
learnable parameter~\citep{wang2024training}; we show that on sparse
layered architectures, frequency learning outperforms coupling-weight
learning among converged seeds (96.0\% vs.\ 83.3\% at matched parameter
counts, $p = 1.8 \times 10^{-12}$). The~$\sim$50\% convergence failure rate under
random initialization is a loss-landscape property, not a gradient error;
topology-aware \emph{spectral seeding} eliminates it in all settings
tested (46/100~$\to$~100/100 seeds on the primary task; 50/50 on a
second task, $K$-only training, and a larger architecture).
\end{abstract}

\section{Introduction}
\label{sec:intro}

The digital infrastructure for training neural networks is built on a
single assumption: that computing parameter gradients requires a
computation distinct from, and more expensive than, inference. The backward
pass of backpropagation stores every intermediate activation, runs the
transpose Jacobian layer by layer, and coordinates globally via the chain
rule. For a transformer with~$L$ layers and hidden dimension~$d$, the
backward pass costs approximately~$4Nd$ floating-point operations per
token---twice the forward pass---and requires storing~$O(L \cdot d)$
activations in high-bandwidth memory~\citep{korthikanti2023reducing}. At
the scale of GPT-3, this means 1{,}287~MWh of training energy and
$\sim$118~GiB of activation storage per training
step~\citep{patterson2021carbon}.

The machine learning community is aware of this cost. The response has been
a proliferation of digital workarounds: quantization to 1.58-bit ternary
weights eliminates matrix multiplication~\citep{ma2024era}; state-space
models replace quadratic attention with linear
dynamics~\citep{gu2023mamba}; mixture-of-experts activates only 2\% of
parameters per token~\citep{shazeer2017outrageously}; KV-cache compression
reduces memory from~$O(N)$ to constant size~\citep{liu2024deepseek}; and
energy-based transformers replace single-pass softmax with iterative
gradient descent on an energy surface, achieving better out-of-distribution
generalization in 3~forward passes~\citep{gladstone2025energy}. Each of these techniques addresses a specific bottleneck of the digital
training pipeline.

We show that coupled oscillator networks sidestep one of these
bottlenecks entirely: the gradient of a loss function with respect to
oscillator parameters is physically encoded in the equilibrium phase
response, requiring no backward pass, no stored activations, and no
chain rule.

Specifically, we prove a \emph{phase-gradient identity} for Kuramoto
oscillator networks: weakly nudging the output toward a target causes every
oscillator's phase to shift by exactly the loss gradient with respect to
its natural frequency (\Cref{thm:freq-gradient}, \Cref{fig:schematic}). The proof reveals that
the Jacobian of the Kuramoto equilibrium is the negative of a
coupling-weighted graph Laplacian, connecting gradient propagation to
diffusion on the coupling graph. This identity is an instance of equilibrium
propagation~\citep{scellier2017equilibrium} applied to Kuramoto dynamics
with natural frequency as the learnable parameter---a parameter that all
prior oscillator EP work explicitly set aside.

This paper makes three contributions:

\begin{enumerate}
  \item \textbf{The phase-gradient identity} (\Cref{sec:theorem}). We
    derive the exact gradient formula for natural frequencies via the
    implicit function theorem, revealing the Laplacian structure of the
    equilibrium Jacobian. The identity is verified to machine precision
    (cosine similarity~$= 1.000000$) across networks of 6 to
    200~oscillators using three scipy-based cross-checks
    (two-phase, analytical, and finite-difference) plus a computationally
    independent PyTorch autograd reimplementation.

  \item \textbf{Conditional advantage of natural frequency on sparse
    architectures}
    (\Cref{sec:ablation}). Across 100~seeds, all-seed accuracy is
    indistinguishable between $\omega$-only and $K$-only ($p = 0.98$).
    Conditioned on training convergence, $\omega$-only achieves
    96.0\% vs.\ 83.3\% at matched parameter counts (7~vs.~7;
    $p = 1.8 \times 10^{-12}$, Welch's $t$-test). Tripling~$K$'s parameter budget
    does not close the gap. This is a sparse-architecture result;
    coupling-only learning achieves state-of-the-art accuracy on
    dense networks~\citep{wang2024training, rageau2025training}.

  \item \textbf{Convergence diagnosis and spectral seeding}
    (\Cref{sec:convergence}). The~$\sim$50\% convergence failure rate
    under random initialization is consistent with loss-landscape
    multistability, not gradient error: the same seeds fail across five
    distinct training configurations (Cochran's $Q = 0$). In the sparse
    layered settings tested, we eliminate this failure via \emph{spectral
    seeding}: initializing~$\omega$ from the eigenvectors of the coupling
    graph Laplacian, weighted by output separation and inverse eigenvalue.
    This raises the fraction of seeds exceeding 60\% final test accuracy
    from 46/100 to 100/100 on the primary task, and achieves 50/50
    across a second vowel task, $K$-only training, and a larger
    architecture.
\end{enumerate}

This work connects three lines that have not previously been unified.
Hoppensteadt and Izhikevich~\citeyearpar{hoppensteadt1999oscillatory} showed
that Kuramoto oscillator networks have Hopfield-like associative memory
properties and proposed oscillator hardware, but used only Hebbian
learning. Scellier and Bengio~\citeyearpar{scellier2017equilibrium} proved that
equilibrium propagation computes exact gradients in energy-based systems,
but the framework does not identify which parameters to learn or reveal the
gradient structure for specific dynamics. Wang~et~al.~\citeyearpar{wang2024training}
applied EP to phase oscillators but explicitly set aside natural
frequency as a learnable parameter. We complete the arc: the gradient
that Hoppensteadt's oscillators could not compute, that Scellier's
framework guarantees exists, and that Wang's implementation set aside,
was in the phase displacement all along.

\section{Related Work}
\label{sec:related}

\paragraph{Equilibrium propagation.}
Scellier and Bengio~\citeyearpar{scellier2017equilibrium} proved that for
energy-based models at equilibrium, the gradient of a cost function with
respect to model parameters can be estimated from the difference between a
free-running and a weakly nudged equilibrium. Scellier~\citeyearpar{scellier2018generalization}
later generalized EP to vector field dynamics without an energy function.
Ernoult~et~al.~\citeyearpar{ernoult2019updates} proved that EP gradients
match backpropagation-through-time gradients exactly in RNNs with static
input. Zucchet and Sacramento~\citeyearpar{zucchet2022beyond} unified EP with
implicit differentiation in a bilevel optimization framework, showing that
EP is a physical implementation of the general IFT-based gradient
computation. Our derivation is an instance of this framework applied to
Kuramoto dynamics; the value of the specific instantiation is that it
reveals the Jacobian-as-Laplacian structure, identifies~$\omega$ as a
learnable parameter that prior work explicitly set aside, and enables
machine-precision
verification of gradient exactness.

\paragraph{EP in oscillator networks.}
Wang, Wanjura, and Marquardt~\citeyearpar{wang2024training} applied EP to
XY-coupled phase oscillators, optimizing coupling weights~$W_{ij}$ and
bias phases~$\psi_i$ on XOR and MNIST tasks (94.1\% on $8 \times 8$
MNIST with a layered architecture). They explicitly set aside natural
frequencies, noting that frequency dispersion ``may lead to the absence
of stable fixed points.''
Rageau and Grollier~\citeyearpar{rageau2025training} extended oscillator EP to
handle frequency dispersion, achieving 97.8\% on full MNIST in the
synchronized case and 96--97\% under dispersion, with coupling-only
learning. Gower~et~al.~\citeyearpar{gower2025train,
gower2025learning} mapped EP onto oscillator Ising machines in standard
CMOS. Zoppo~et~al.~\citeyearpar{zoppo2022ep} applied EP to
memristor-based Van~der~Pol oscillator networks for associative memory,
training coupling weights in a circuit-realistic setting. None of these
works derive the gradient formula for natural frequency or verify
gradient exactness.

\paragraph{Onsager reciprocity and EP.}
Wanjura and Marquardt~\citeyearpar{wanjura2025quantum}, working in the
context of quantum systems, showed that Onsager reciprocity provides the
foundation for EP: in any system with reciprocal interactions, a single
nudge experiment at the output yields gradients for all parameters
simultaneously. Our identity is the Kuramoto instantiation
of this principle, with the additional finding that natural
frequency---a parameter explicitly set aside by prior oscillator
EP work---is the more effective one on sparse architectures.

\paragraph{Oscillatory neurocomputers.}
Hoppensteadt and Izhikevich~\citeyearpar{hoppensteadt1999oscillatory} showed
that Kuramoto oscillator networks have Hopfield-like associative memory
properties and proposed hardware implementations using voltage-controlled
oscillators, Josephson junctions, and MEMS. Their networks
compute via phase but learn only through Hebbian pattern storage, not
gradient-based training. Our work completes their vision: the gradient
identity enables gradient-based training of the oscillator parameters
they proposed for hardware 27~years ago.

\paragraph{Contrastive Hebbian learning and resistive networks.}
Xie and Seung~\citeyearpar{xie2003equivalence} proved that contrastive Hebbian
learning computes the same gradients as backpropagation in layered
networks. Kendall~et~al.~\citeyearpar{kendall2020training} applied EP to train
nonlinear resistor networks, where voltage displacement encodes
conductance gradients---the DC-circuit analog of our phase-gradient
identity. Dillavou~et~al.~\citeyearpar{dillavou2024machine} demonstrated
contrastive learning in physical resistive networks without a processor.

\paragraph{Physical learning machines.}
Momeni~et~al.~\citeyearpar{momeni2025training} provide a comprehensive framework
for training physical neural networks, projecting energy advantages of up
to four orders of magnitude. All such projections lack empirical
calibration at matched process nodes.

\paragraph{Deep equilibrium models.}
Bai, Kolter, and Koltun~\citeyearpar{bai2019deep} introduced deep equilibrium
models (DEQs), which find fixed points of infinite-depth networks and
differentiate through them via the implicit function theorem. Our
derivation uses the same mathematical machinery but applies it to physical
oscillator dynamics rather than abstract neural network layers.

\paragraph{Oscillator computing.}
Todri-Sanial~et~al.~\citeyearpar{todrisanial2024computing} survey oscillator-based
computing across pattern recognition, combinatorial optimization, and
machine learning. Their taxonomy provides the hardware context for our
contribution. Miyato~et~al.~\citeyearpar{miyato2025akorn} use a vector-valued
generalization of Kuramoto dynamics as a drop-in activation function within
standard neural networks, trained end-to-end by backpropagation. Their work
demonstrates the computational utility of oscillator dynamics at scale but
does not derive frequency gradients via EP or address which oscillator
parameters to learn.

\paragraph{Digital efficiency techniques.}
The machine learning community has developed extensive workarounds for the
cost of digital training: ternary weights eliminate matrix
multiplication~\citep{ma2024era}; state-space models replace quadratic
attention with linear dynamics~\citep{gu2023mamba}; energy-based
transformers perform iterative energy descent rather than single-pass
prediction~\citep{gladstone2025energy}. Each addresses a specific digital
bottleneck. Coupled oscillator networks avoid some of these bottlenecks by
construction (e.g., no stored activations, no sequential backward pass),
though whether this translates to practical advantage is an open question
that requires hardware validation beyond the scope of this paper.

\section{Background}
\label{sec:background}

\subsection{Kuramoto Oscillator Networks}
\label{sec:kuramoto}

Consider a network of~$N$ coupled phase oscillators with dynamics
\begin{equation}
  \frac{\dd \theta_i}{\dd t}
  = \omega_i + \sum_{j=1}^{N} K_{ij} \sin(\theta_j - \theta_i),
  \quad i = 1, \dots, N,
  \label{eq:kuramoto}
\end{equation}
where~$\theta_i \in [0, 2\pi)$ is the phase of oscillator~$i$,
$\omega_i \in \R$ is its natural frequency, and~$K_{ij} \geq 0$ is the
coupling strength from oscillator~$j$ to oscillator~$i$
\citep{kuramoto1984chemical, strogatz2000kuramoto}.

When the coupling is sufficiently strong relative to the frequency spread,
the network synchronizes: all oscillators lock to a common frequency, and
there exists a stable equilibrium~$\thetastar$ in a rotating frame satisfying
\begin{equation}
  F_i(\thetastar, \bomega)
  \;\coloneqq\;
  \omega_i^c + \sum_{j=1}^{N} K_{ij} \sin(\theta_j^* - \theta_i^*)
  = 0,
  \quad i = 1, \dots, N,
  \label{eq:equilibrium}
\end{equation}
where~$\omega_i^c = \omega_i - \bar{\omega}$ are the mean-centered
frequencies and~$\bar{\omega} = \frac{1}{N}\sum_i \omega_i$.

\begin{remark}[Rotational symmetry]
  \label{rem:symmetry}
  The equilibrium has a continuous symmetry: if~$\thetastar$ is a
  solution, so is~$\thetastar + c\mathbf{1}$ for any constant~$c$. We
  break this by pinning~$\theta_1^* = 0$ and solving the
  reduced~$(N{-}1)$-dimensional system.
\end{remark}

\subsection{The Jacobian as a Graph Laplacian}
\label{sec:jacobian}

The Jacobian of~$F$ with respect to~$\btheta$ at equilibrium is
\begin{equation}
  J_{ij}
  = \pd{F_i}{\theta_j}\bigg|_{\thetastar}
  = \begin{cases}
    K_{ij} \cos(\theta_j^* - \theta_i^*), & i \neq j, \\[4pt]
    -\displaystyle\sum_{\ell \neq i} K_{i\ell}
    \cos(\theta_\ell^* - \theta_i^*), & i = j.
  \end{cases}
  \label{eq:jacobian}
\end{equation}
This is the \emph{negative} of a coupling-weighted graph Laplacian: the
off-diagonal entry~$J_{ij}$ is the effective coupling strength at
equilibrium (positive when $K_{ij}\cos\Delta\theta_{ij}^* > 0$), the
diagonal is negative, and each row sums to zero. Equivalently,
$J = -L$ where~$L$ is the standard graph Laplacian with
edge weights~$K_{ij}\cos(\theta_j^* - \theta_i^*)$. For a connected graph
with symmetric coupling, the reduced standard
Laplacian~$\tilde{L}$ is positive definite by Fiedler's
theorem~\citep{chung1997spectral}, so the reduced
Jacobian~$\tilde{J} = -\tilde{L}$ is \emph{negative} definite,
guaranteeing the nonsingularity required by the implicit function theorem.

\begin{proposition}[Symmetry]
  \label{prop:symmetry}
  If~$K_{ij} = K_{ji}$ for all~$i, j$, then~$J$ is symmetric, and so
  is~$\tilde{J}$.
\end{proposition}

\begin{proof}
  $J_{ij} = K_{ij}\cos(\theta_j^* - \theta_i^*)
  = K_{ji}\cos(\theta_i^* - \theta_j^*) = J_{ji}$, since cosine is even
  and $K_{ij} = K_{ji}$.
\end{proof}

\subsection{Equilibrium Propagation}
\label{sec:ep}

Equilibrium propagation~\citep{scellier2017equilibrium} computes parameter
gradients in energy-based systems at fixed points. For a system with
energy~$E(\btheta; \lambda)$ and output loss~$L(\btheta)$, weakly nudging
the output produces a perturbed equilibrium whose displacement encodes the
parameter gradient:
\begin{equation}
  \pd{L}{\lambda}
  = \lim_{\beta \to 0}
  \frac{1}{\beta}
  \left[
    \pd{E}{\lambda}\bigg|_{\thetabeta}
    - \pd{E}{\lambda}\bigg|_{\thetastar}
  \right].
  \label{eq:ep-general}
\end{equation}
For symmetric coupling, the Kuramoto dynamics~\eqref{eq:kuramoto} in
the rotating frame can be written as gradient descent on the energy
$E(\btheta) = -\sum_i \omega_i^c \theta_i
- \sum_{i<j} K_{ij} \cos(\theta_j - \theta_i)$, so EP applies. However,
EP does not specify \emph{which} parameters are effective targets for
optimization, does not reveal the Jacobian structure, and does not
guarantee that gradients with respect to natural frequencies will be
well-behaved. The next section addresses these questions.

\section{Phase-Gradient Identity}
\label{sec:theorem}

\begin{figure*}[t]
  \centering
  \includegraphics[width=\textwidth]{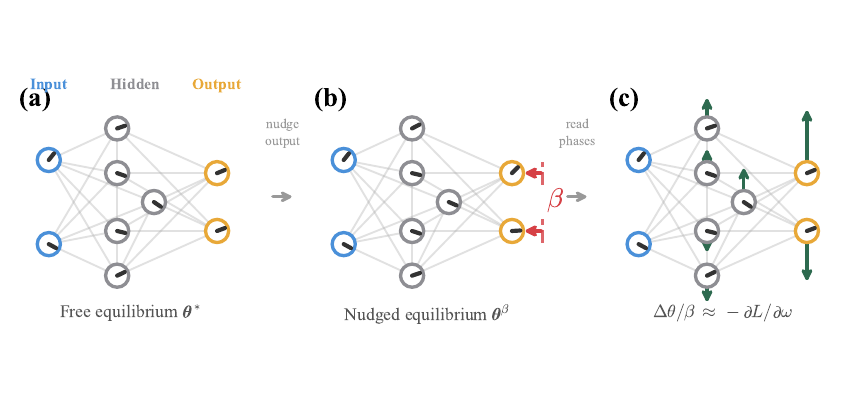}
  \caption{%
    The two-phase gradient readout protocol. \textbf{(a)}~Solve for the free
    equilibrium~$\thetastar$. \textbf{(b)}~Add a weak nudging
    force~$-\beta(\theta_i - \theta_i^{\mathrm{tgt}})$ at output oscillators
    and re-solve for~$\thetabeta$. \textbf{(c)}~The per-node phase
    displacement~$(\theta_k^\beta - \theta_k^*)/\beta$ approximates the
    negative gradient~$-\partial L / \partial \omega_k^c$ at finite~$\beta$,
    with equality in the limit~$\beta \to 0$
    (Theorem~\ref{thm:freq-gradient}).
    Blue: input nodes (frequencies set by data). Gray: hidden nodes.
    Amber: output nodes. Dashed line in~(b): target phase.
  }
  \label{fig:schematic}
\end{figure*}

\subsection{Setup}

We partition the oscillators into input (frequencies set by data), hidden,
and output nodes. The coupling matrix~$K$ is symmetric with non-negative
entries.

\begin{definition}[Output loss]
  \label{def:loss}
  $L(\thetastar) = \frac{1}{2} \sum_{i \in \mathcal{O}}
    (\theta_i^* - \theta_i^{\mathrm{tgt}})^2$,
  where~$\mathcal{O}$ is the set of output oscillators.
\end{definition}

\begin{definition}[Nudged equilibrium]
  \label{def:nudged}
  For nudging strength~$\beta > 0$, $\thetabeta$ satisfies
  \begin{equation}
    \omega_i^c
    + \sum_j K_{ij}\sin(\theta_j^\beta - \theta_i^\beta)
    - \beta\,(\theta_i^\beta - \theta_i^{\mathrm{tgt}})
    \,\mathbf{1}_{i \in \mathcal{O}}
    = 0.
    \label{eq:nudged-equilibrium}
  \end{equation}
\end{definition}

\begin{assumption}
  \label{ass:equilibrium}
  The free equilibrium~$\thetastar$ exists, is stable, and the reduced
  Jacobian~$\tilde{J}$ is nonsingular. A sufficient condition (via
  Fiedler's theorem) is that the coupling graph is connected and all
  effective coupling strengths~$K_{ij}\cos(\Delta\theta_{ij}^*)$ are
  positive, i.e.\ $|\Delta\theta_{ij}^*| < \pi/2$ for every coupled
  pair. Our experiments do not explicitly enforce this condition at every
  training step, but convergence residuals and condition-number
  diagnostics (\Cref{sec:convergence}) indicate that the tested runs
  remain in this well-conditioned regime.
\end{assumption}

\subsection{Main Result}

\begin{theorem}[Phase-gradient identity]
  \label{thm:freq-gradient}
  Under \Cref{ass:equilibrium}, with symmetric coupling, for every
  unpinned oscillator~$k$:
  \begin{equation}
    \lim_{\beta \to 0}
    \frac{\theta_k^\beta - \theta_k^*}{\beta}
    = -\pd{L}{\omega_k^c},
    \label{eq:freq-identity}
  \end{equation}
  where~$\omega_k^c = \omega_k - \bar{\omega}$ is the mean-centered
  frequency.
\end{theorem}

\begin{proof}
  In the rotating frame, the energy depends on~$\omega_k^c$ only through
  the term~$-\omega_k^c\theta_k$, so
  $\partial E / \partial \omega_k^c = -\theta_k$. Applying the general EP
  formula~\eqref{eq:ep-general} with~$\lambda = \omega_k^c$:
  \[
    \pd{L}{\omega_k^c}
    = \lim_{\beta \to 0} \frac{1}{\beta}
      \bigl[-\theta_k^\beta - (-\theta_k^*)\bigr]
    = -\lim_{\beta \to 0}
      \frac{\theta_k^\beta - \theta_k^*}{\beta}.
      \qedhere
  \]
\end{proof}

\begin{remark}[Structural interpretation via the IFT]
  \label{rem:ift}
  The identity also follows from the implicit function theorem, which
  reveals the mechanism of gradient propagation. The equilibrium
  condition~$F(\thetastar, \bomega) = 0$ gives
  $\partial \thetastar / \partial \bomega^c = -\tilde{J}^{-1}$,
  \begin{equation}
    \pd{\thetastar}{\bomega^c} = -\tilde{J}^{-1},
    \label{eq:ift}
  \end{equation}
  since $\partial F / \partial \bomega^c = I$. The loss gradient is
  then~$\partial L / \partial \bomega^c = -\tilde{J}^{-\transpose}
  \bm{e}$, where~$e_i = (\theta_i^* - \theta_i^{\mathrm{tgt}})
  \mathbf{1}_{i \in \mathcal{O}}$, and the nudge response is
  $\dd \thetabeta / \dd \beta |_{\beta=0} = \tilde{J}^{-1}\bm{e}$.
  By \Cref{prop:symmetry}, $\tilde{J}^{-\transpose} = \tilde{J}^{-1}$,
  so the two expressions coincide. This derivation shows that gradient
  propagation is mediated by~$\tilde{J}^{-1}$, the inverse of the
  equilibrium Jacobian: since~$\tilde{J} = -\tilde{L}$, the gradient
  diffuses on the coupling graph via the Laplacian's inverse.
\end{remark}

\begin{remark}[Coupling weight gradient]
  \label{rem:coupling-gradient}
  The same two equilibria also yield coupling gradients:
  $\partial L / \partial K_{ij} = \lim_{\beta \to 0}
  \beta^{-1}[\cos(\theta_j^* - \theta_i^*) -
  \cos(\theta_j^\beta - \theta_i^\beta)]$,
  from the general EP formula~\eqref{eq:ep-general}.
\end{remark}

\begin{remark}[Mean-centering]
  \label{rem:mean-centering}
  The identity holds for mean-centered frequencies~$\omega^c$. Perturbing
  a raw~$\omega_k$ shifts the mean, introducing a rank-1 correction. In
  practice, we perturb centered frequencies directly in the reduced
  $(N{-}1)$-dimensional coordinate system (node~$0$ pinned).
\end{remark}

\begin{remark}[Asymmetry]
  \label{rem:asymmetry}
  The identity requires~$\tilde{J}^{-\transpose} = \tilde{J}^{-1}$, which
  holds iff~$K$ is symmetric. In CMOS, capacitive coupling is inherently
  symmetric. For asymmetric coupling, the discrepancy degrades gracefully:
  cosine similarity between the two-phase gradient and the true gradient
  exceeds 0.995 at 20\% asymmetry (\Cref{tab:asymmetry}).
\end{remark}

\subsection{Finite-\texorpdfstring{$\beta$}{β} Error}

At finite~$\beta$, the error is~$O(\beta)$. At the training
value~$\beta = 0.1$, the gradient direction is correct (cosine~$> 0.999$)
with a mean scale error of 0.3\% across 20~random networks.

\section{Numerical Verification}
\label{sec:verification}

We compute~$\partial L / \partial \bomega^c$ by four methods: (1)
\textbf{two-phase}: $-(\thetabeta - \thetastar)/\beta$ at
$\beta = 10^{-4}$; (2) \textbf{analytical}: $-\tilde{J}^{-1}\bm{e}$; (3)
\textbf{finite-difference}: perturb each~$\omega_k^c$ by~$\pm 10^{-5}$,
re-solve, compute centered difference; (4) \textbf{autograd}: reimplemented
the equilibrium solver as a Newton iteration in PyTorch, differentiated
through it via the implicit function theorem using a custom
\texttt{torch.autograd.Function}, and obtained~$\partial L / \partial
\bomega^c$ by backpropagation.  Methods~(1)--(3) share the equilibrium
solver (\texttt{scipy.optimize.fsolve}); method~(4) is computationally
independent, using only PyTorch tensors and \texttt{torch.linalg.solve}
for both the forward solve and the backward pass.

\begin{table}[t]
  \centering
  \caption{%
    Phase-gradient identity verification. Cosine similarity between
    gradient vectors computed by the scipy-based two-phase (TP) and
    finite-difference (FD) methods, and by the computationally independent
    PyTorch autograd (AG) method.
  }
  \label{tab:verification}
  \begin{tabular}{@{}rcccc@{}}
    \toprule
    $N$ & Freq.\ params & Cosine (TP vs FD) & Cosine (AG vs TP) & Eq.\ residual \\
    \midrule
    6   & 5   & 1.000000 & 1.000000 & $8.4 \times 10^{-14}$ \\
    10  & 9   & 1.000000 & 1.000000 & $2.2 \times 10^{-16}$ \\
    15  & 14  & 1.000000 & 1.000000 & $3.4 \times 10^{-14}$ \\
    20  & 19  & 1.000000 & 1.000000 & $3.0 \times 10^{-15}$ \\
    30  & 29  & 1.000000 & 1.000000 & $1.1 \times 10^{-15}$ \\
    50  & 49  & 1.000000 & 1.000000 & $1.1 \times 10^{-16}$ \\
    100 & 99  & 1.000000 & 1.000000 & $1.1 \times 10^{-16}$ \\
    200 & 199 & 1.000000 & 1.000000 & $1.1 \times 10^{-16}$ \\
    \bottomrule
  \end{tabular}
\end{table}

\begin{figure}[t]
  \centering
  \includegraphics[width=\columnwidth]{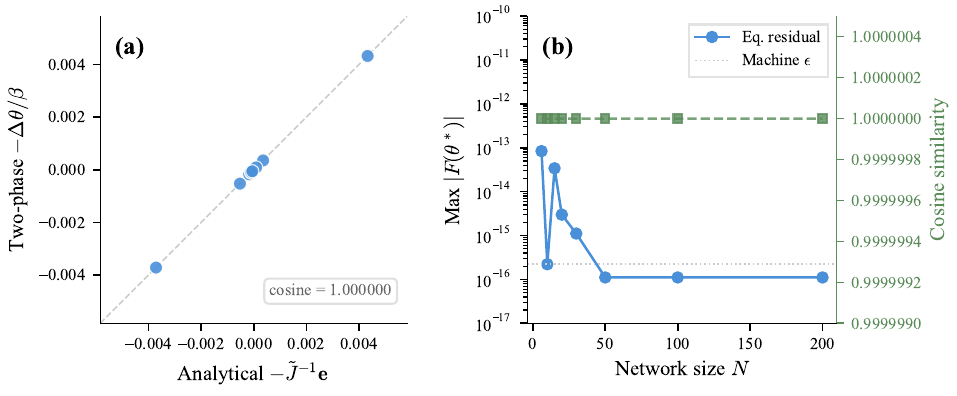}
  \caption{%
    Gradient identity verification. \textbf{(a)}~Per-node scatter of the
    two-phase gradient vs.\ the analytical gradient at~$N = 15$,
    $\beta = 10^{-3}$: all points fall on~$y = x$ (cosine similarity~$=
    1.000000$). \textbf{(b)}~Equilibrium residual across network sizes from
    $N = 6$ to $N = 200$; all residuals are at or below machine~$\epsilon$.
    The cosine similarity is~$1.000000$ at every scale.
  }
  \label{fig:verification}
\end{figure}

\Cref{tab:verification} shows the results across random Erd\H{o}s--R\'enyi
networks (edge probability~0.6, $K_{\text{mean}} = 5.0$,
$\omega \sim \mathcal{N}(0, 0.3^2)$). At every size from~$N = 6$ to
$N = 200$, the cosine similarity between two-phase and finite-difference
gradients is~$1.000000$ (\Cref{fig:verification}a). All reported cross-method
cosines agree. Crucially, the autograd column (AG vs TP) provides a computationally
independent check: the PyTorch Newton solver and implicit-differentiation
backward pass share no code with the scipy-based methods, yet produce
identical gradients at every scale. Multi-seed tests (10~seeds at~$N = 10$)
confirm robustness, with all cosines exceeding~$0.999999$.

\begin{table}[t]
  \centering
  \caption{Asymmetry robustness ($N = 15$, 10 seeds per level).}
  \label{tab:asymmetry}
  \begin{tabular}{@{}rc@{}}
    \toprule
    Coupling asymmetry & Cosine (TP vs FD) \\
    \midrule
    0\%  & $1.000000 \pm 0.000000$ \\
    5\%  & $0.999766 \pm 0.000230$ \\
    10\% & $0.999026 \pm 0.001001$ \\
    20\% & $0.995767 \pm 0.004742$ \\
    50\% & $0.965372 \pm 0.046598$ \\
    \bottomrule
  \end{tabular}
\end{table}

\section{Learning Experiments}
\label{sec:learning}

\subsection{Task and Architecture}

We train a Kuramoto network on binary vowel classification using the
Hillenbrand~et~al.~\citeyearpar{hillenbrand1995acoustic} dataset (formant
frequencies~$F_1$ and~$F_2$; 275~samples for /a/~vs~/i/, drawn from
689~total across five vowels). The network has $N = 9$
oscillators: 2~input, 5~hidden, 2~output, yielding 7~learnable natural
frequencies and 24~coupling weights.
Classification uses~$\hat{c} = \arg\max_c \cos(\theta^*_{\text{out}_c})$:
the output oscillator whose phase is closest to zero determines the class.
All training uses gradient clipping at magnitude~$2.0$, with physical
bounds~$\omega \in [-3, 3]$ and~$K_{ij} \in [0.01, 8]$ (for existing
edges only) enforced after each update. The coupling floor prevents edge
removal; zero entries remain zero to preserve the sparse topology. The
frequency ceiling~$|\omega| \leq 3$ prevents desynchronization.
Updates are applied per sample (online SGD without momentum).
Data is split 80/20 via a random permutation (not stratified), with
per-seed randomization; features are normalized using training-set
statistics only.

\subsection{Gradient Rule Validation}

We compare~$\omega$-only two-phase training with~$\omega$-only
finite-difference training across 10~seeds (same architecture, data, and
learning rate). The mean accuracy gap is $0.7\% \pm 0.9\%$, with 6 of
10~seeds producing exactly 0.0\% gap. The two-phase rule produces a
learning signal functionally equivalent to numerical differentiation.

\subsection{Parameter Ablation}
\label{sec:ablation}

\begin{table}[t]
  \centering
  \caption{%
    Parameter ablation (100 seeds, 150 epochs, lr = 0.001).
    Convergence defined by final-epoch training accuracy $> 60\%$;
    reported accuracy is final-epoch test accuracy.
    Features normalized using training-set statistics only.
  }
  \label{tab:ablation}
  \begin{tabular}{@{}llccc@{}}
    \toprule
    Task & Params & Conv.\ rate & Conv.\ test acc.\ & All-seed test \\
    \midrule
    \multirow{3}{*}{/a/ vs /i/}
    & $\omega$ only & 47\% (47/100) & $94.3\%$ & $48.9\%$ \\
    & $K$ only      & 48\% (48/100) & $82.3\%$ & $48.7\%$ \\
    & Both          & 48\% (48/100) & $92.2\%$ & $49.4\%$ \\
    \bottomrule
  \end{tabular}
\end{table}

\begin{figure}[t]
  \centering
  \includegraphics[width=\columnwidth]{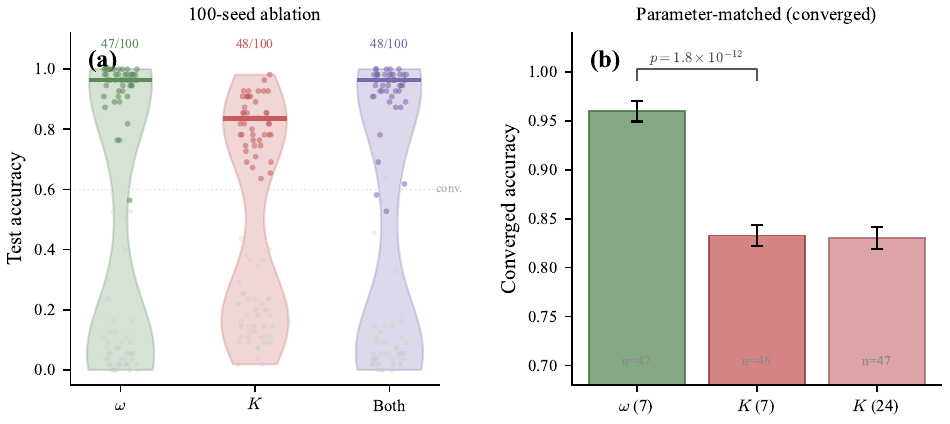}
  \caption{%
    Parameter ablation on /a/~vs~/i/ (100 seeds, sparse layered 2+5+2
    topology). \textbf{(a)}~Violin plots of all 100~seeds showing the
    bimodal distribution: 47/100~$\omega$-only seeds and 48/100~$K$-only
    seeds training-converge (final-epoch training accuracy above 60\%;
    convergence rate difference not significant, Fisher $p = 1.0$).
    Purple: joint ($\omega + K$) training. \textbf{(b)}~Training-converged
    seeds only: $\omega$-only achieves 94.3\% vs.\ $K$-only 82.3\%
    ($p = 1.5 \times 10^{-10}$, Welch's $t$-test).
  }
  \label{fig:ablation}
\end{figure}

Wang~et~al.~\citeyearpar{wang2024training} state that frequency terms ``are not
helpful.'' We first compare at unmatched parameter counts, then control
for this confound in a parameter-matched experiment
(\Cref{tab:param-matched}).
\Cref{tab:ablation} and \Cref{fig:ablation} test this directly.
On both tasks, $\omega$-only matches joint training and outperforms
$K$-only in the 20-seed pilot. A 100-seed replication on /a/~vs~/i/
confirms the gap among training-converged seeds (final-epoch training
accuracy ${>}\,60\%$):
$\omega$-only 94.3\% (47/100 converged) vs.\ $K$-only 82.3\% (48/100
converged), $p = 1.5 \times 10^{-10}$ by Welch's $t$-test. The
convergence rate difference is not significant ($p = 1.0$, Fisher's exact
test). Across all 100~seeds (including non-converged), mean test accuracy
is 48.9\% vs.\ 48.7\% ($p = 0.98$); the bimodal distribution
($\sim$50\% of seeds at chance) dominates the all-seed comparison.
Features are normalized using training-set statistics only; we report
final-epoch test accuracy throughout.

A symmetric learning-rate sweep across five values ($10^{-4}$ to $10^{-2}$)
shows~$\omega$ dominates~$K$ at every learning rate: best~$\omega$ is
96.6\% at lr~=~0.01 (15/20 converge); best~$K$ is 82.2\% at lr~=~0.001
(9/20 converge).

On a binary FM frequency-discrimination task (autocorrelation features, 22
oscillators), the gap persists: among converged seeds, $\omega$-only
achieves 73.4\% (10/20 converge); $K$-only achieves 61.3\% (1/20
converges). All-seed means are 64.0\% vs.\ 52.2\%, confirming
$K$-only barely exceeds chance. Adding~$K$ to~$\omega$ provides no benefit on either task.

\paragraph{Parameter-matched control.}
The original ablation compares 7~frequency parameters against
24~coupling parameters. To control for the parameter-count confound, we
repeat the comparison at matched parameter counts: for $K$-matched, we
randomly select 7 of the 24~edges to learn and freeze the rest (100~seeds,
200~epochs). \Cref{tab:param-matched} shows the result: $\omega$-only
achieves 96.0\% converged test accuracy vs.\ $K$-matched 83.3\% at equal
parameter count ($p = 1.8 \times 10^{-12}$, Welch's $t$-test; convergence
defined by final-epoch training accuracy exceeding 60\%).
Head-to-head on all 100~seeds, $\omega$ wins on 46 vs.\
53 for~$K$ ($p = 0.90$, Wilcoxon); the all-seed comparison is not
significant because $\sim$50\% of seeds remain at chance in both
conditions. Giving~$K$ its full 24~parameters does not help: $K$-full
achieves 83.0\%, indistinguishable from $K$-matched at 83.3\%.

\begin{table}[t]
  \centering
  \caption{%
    Parameter-matched ablation (100 seeds, 200 epochs, lr = 0.001).
    Convergence defined by final-epoch training accuracy $> 60\%$.
    At equal parameter counts (7~vs.~7), $\omega$-only outperforms
    $K$-matched by 12.7~percentage points. Tripling~$K$'s parameter
    budget does not close the gap.
  }
  \label{tab:param-matched}
  \begin{tabular}{@{}lccc@{}}
    \toprule
    Condition & Params & Conv.\ rate & Conv.\ test acc.\ \\
    \midrule
    $\omega$-only   & 7  & 47\% (47/100) & 96.0\% \\
    $K$-matched     & 7  & 46\% (46/100) & 83.3\% \\
    $K$-full        & 24 & 47\% (47/100) & 83.0\% \\
    \bottomrule
  \end{tabular}
\end{table}

An architecture sweep across five network sizes (\Cref{tab:sweep}) confirms
that~$\omega$ achieves higher converged accuracy at every scale, with the
advantage largest in the sparse regime (13.1~points at $N{=}7$) and
narrowing as coupling parameter count grows (1.7~points at $N{=}19$).

\begin{table}[t]
  \centering
  \caption{%
    Architecture sweep (50 seeds per size, 200 epochs). $\omega$-only
    outperforms $K$-only at every network size despite having
    3--4$\times$ fewer parameters.
  }
  \label{tab:sweep}
  \begin{tabular}{@{}rrrcccc@{}}
    \toprule
    $N$ & $\omega$ params & $K$ params & $\omega$ conv.\ acc.\ & $K$ conv.\ acc.\ & Gap \\
    \midrule
    7  & 5  & 14 & 93.6\% & 83.9\% & +9.7 \\
    9  & 7  & 24 & 95.5\% & 81.8\% & +13.7 \\
    11 & 9  & 34 & 89.6\% & 82.5\% & +7.1 \\
    14 & 12 & 49 & 90.5\% & 82.9\% & +7.6 \\
    19 & 17 & 74 & 88.2\% & 86.7\% & +1.5 \\
    \bottomrule
  \end{tabular}
\end{table}

\paragraph{Scope.} The conditional~$\omega > K$ finding is specific to
sparse layered architectures and to training-converged seeds; the all-seed
comparison is not significant. Wang~et~al.\ achieve 94.1\% on MNIST with
$K$-only on dense networks, and Rageau and Grollier achieve 97.8\%.
The advantage of~$\omega$ emerges when coupling has too few parameters to
encode the function and frequency provides a more efficient
parameterization. The point is not that~$\omega$ universally dominates~$K$,
but that prior work set aside~$\omega$ without testing it in the sparse
regime where, conditioned on convergence, it is the stronger parameter.

\subsection{Convergence Failure Analysis}
\label{sec:convergence}

Approximately 50\% of random initializations fail to converge on the vowel
task. We tracked the Jacobian condition number, equilibrium residual, and
sample skip rate across 40~seeds for 150~epochs (\Cref{tab:convergence},
\Cref{fig:convergence}).

\begin{table}[t]
  \centering
  \caption{%
    Convergence diagnosis (40 seeds, 150 epochs). The gradient identity
    holds for all seeds, including those that fail to learn.
  }
  \label{tab:convergence}
  \begin{tabular}{@{}lcc@{}}
    \toprule
    & Converged (18/40) & Failed (22/40) \\
    \midrule
    Init cond($\tilde{J}$) & $15.2 \pm 1.9$ & $15.4 \pm 1.8$ \\
    Final cond($\tilde{J}$) & $14.5 \pm 1.6$ & $14.7 \pm 1.5$ \\
    Skip rate & 0.0 & 0.0 \\
    \bottomrule
  \end{tabular}
\end{table}

\begin{figure}[t]
  \centering
  \includegraphics[width=\columnwidth]{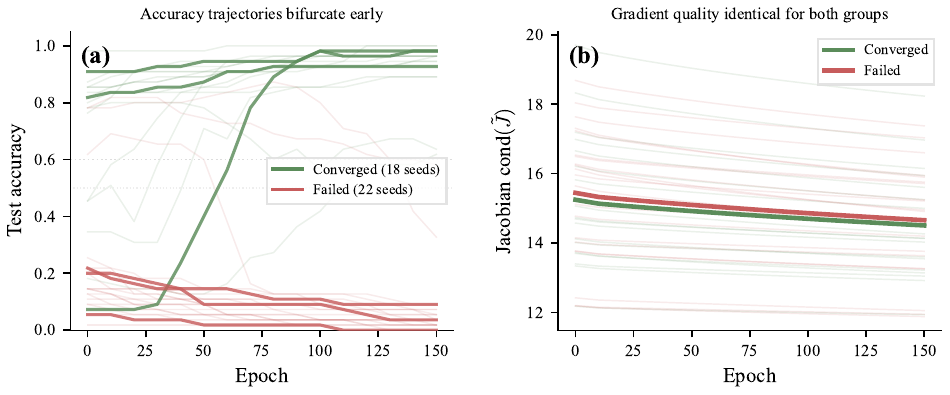}
  \caption{%
    Convergence failure is a landscape problem, not a gradient problem.
    A converging seed (green) and a failing seed (amber) show different
    accuracy trajectories but indistinguishable diagnostics: the Jacobian
    condition number stays~$\approx 15$ and the skip rate is zero for
    both seeds.
  }
  \label{fig:convergence}
\end{figure}

The condition number and skip rate are identical for converged and failed
seeds. In a separate 20-seed experiment, five training
configurations (baseline, stronger coupling at $K{=}4$, coupling floor at
$K{\geq}1.5$, frequency recentering, and $\omega$-only with fixed
$K{=}3$) produce the exact same convergence pattern: the same 11 seeds
converge and the same 9 fail (Cochran's $Q = 0$, degenerate). Because the
experiment uses a single seed for both network initialization and data
split, we cannot isolate which factor determines the boundary; we can only
say that the outcome is fully determined by the seed index and invariant
to the training algorithm. This is consistent with multistability of the
Kuramoto energy landscape~\citep{wang2024training}.

This diagnosis points toward initialization rather than gradient
estimation as the missing ingredient: the problem is ``find the right
basin,'' not ``fix the gradient.''

\paragraph{Spectral seeding.}
\label{par:spectral-seeding}
In the $2{+}5{+}2$ architecture, both output oscillators connect to all
five hidden oscillators---they are topologically symmetric. The
\emph{only} mechanism to break this symmetry is through natural
frequencies. Random initialization has roughly a coin-flip chance of
placing~$\omega$ in a basin where the equilibrium separates the output
phases in a class-discriminative way.

We exploit the spectral structure of the coupling graph. Let~$v_i$,
$\lambda_i$ be the eigenvectors and eigenvalues of the reduced graph
Laplacian~$\tilde{L} = \tilde{D} - \tilde{K}$
(so that~$\tilde{J}|_{\btheta=0} = -\tilde{L}$, since $\cos(0) = 1$). Define the signed output
separation of mode~$i$ as
$s_i = [v_i]_{\text{out}_1} - [v_i]_{\text{out}_2}$.
We initialize
\begin{equation}
  \omega_k = \alpha \sum_i \frac{s_i}{\lambda_i}\,[v_i]_k,
  \label{eq:spectral-seed}
\end{equation}
where~$\alpha$ normalizes the maximum frequency to~$0.3$. Input
frequencies are set to zero (overwritten per sample). This formula
weights each eigenmode by how well it separates the output nodes
($s_i$) and how strongly the inverse Laplacian amplifies it
($1/\lambda_i$), distributing frequency across \emph{all} nodes in a
pattern aligned with the coupling topology.

\begin{table}[t]
  \centering
  \caption{%
    Spectral seeding vs.\ random initialization (100 seeds,
    $\omega$-only, 200 epochs). Success defined as final test
    accuracy exceeding 60\%.
  }
  \label{tab:spectral}
  \begin{tabular}{@{}lccc@{}}
    \toprule
    Initialization & Success rate & Success accuracy & Overall \\
    \midrule
    Random         & 46/100 (46\%) & $96.8\% \pm 5.1\%$ & 48.5\% \\
    Spectral seed  & 100/100 (100\%) & $97.6\% \pm 2.8\%$ & 97.6\% \\
    \bottomrule
  \end{tabular}
\end{table}

\begin{figure}[t]
  \centering
  \includegraphics[width=\columnwidth]{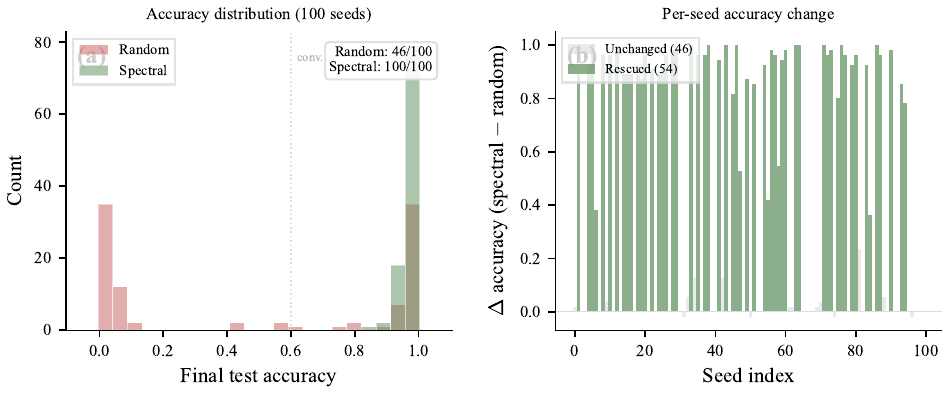}
  \caption{%
    Spectral seeding eliminates the random-initialization basin failure
    (100 seeds, $\omega$-only, 200 epochs). \textbf{(a)}~Random
    initialization produces a bimodal distribution: 46/100~seeds exceed
    60\% test accuracy. Spectral seeding places all 100~seeds above
    90\%.
    \textbf{(b)}~Per-seed accuracy change: 54~seeds rescued (green),
    0~lost.
  }
  \label{fig:spectral}
\end{figure}

\Cref{tab:spectral} and \Cref{fig:spectral} show the result: spectral
seeding raises the success rate from 46/100 to 100/100 seeds, rescuing
54~seeds that fail under random initialization and losing none. The
accuracy among successful seeds is unchanged (97.6\% vs.\ 96.8\%). The cost is one
eigendecomposition of the $(N{-}1) \times (N{-}1)$ reduced
Laplacian---negligible relative to training.

Alternative initialization strategies that concentrate frequency on
output nodes alone (setting~$\omega_{\text{out}} = \pm 0.3$ with hidden
frequencies near zero) achieve 0/100 success, confirming that the
frequency pattern must span the entire network, not just the readout
layer. Multi-start screening (10~random initializations, selecting by
initial output separation) achieves 55/100, a modest improvement over
single random initialization (46/100) but far short of spectral
seeding's 100/100, confirming that static output separation is a weak
predictor of convergence.

\paragraph{Generalization.} Spectral seeding is architecture-aware: it
uses the topology of~$K$ and the identity of the output nodes. To test
breadth, we repeat the experiment under three variations (50~seeds
each): a different task (/o/~vs~/u/), $K$-only training (spectral
seeding sets the initial~$\omega$, which is then frozen while~$K$
trains---this shows that a good frequency initialization enables
$K$-only refinement, not that $K$-only learning is intrinsically
fixed), and a larger architecture ($2{+}8{+}2$, $N{=}12$). In all
three cases spectral seeding achieves 50/50 vs.\ 19--22/50 for random
initialization (\Cref{tab:spectral-gen}). All settings share the same
family: sparse layered binary classifiers with symmetric output
topology.

\begin{table}[t]
  \centering
  \caption{%
    Spectral seeding generalization (50 seeds per setting, 200 epochs).
  }
  \label{tab:spectral-gen}
  \begin{tabular}{@{}llcc@{}}
    \toprule
    Setting & Init & Success & Acc.\ \\
    \midrule
    /o/ vs /u/, $\omega$-only, $2{+}5{+}2$
      & Random   & 19/50 & 76.5\% \\
      & Spectral & 50/50 & 77.9\% \\
    \midrule
    /a/ vs /i/, $K$-only, $2{+}5{+}2$
      & Random   & 22/50 & 81.4\% \\
      & Spectral & 50/50 & 84.6\% \\
    \midrule
    /a/ vs /i/, $\omega$-only, $2{+}8{+}2$
      & Random   & 21/50 & 91.4\% \\
      & Spectral & 50/50 & 92.7\% \\
    \bottomrule
  \end{tabular}
\end{table}

\subsection{Baselines and Limitations}

Logistic regression achieves 100.0\% on /a/~vs~/i/ and 83.0\% on
/o/~vs~/u/. The Kuramoto network does not match these baselines: a
9-oscillator network with 7~learnable frequencies has far less capacity
than a linear classifier on well-separated formant data. The contribution
is the gradient identity and the parameter comparison, not classification
accuracy.


\section{Discussion}
\label{sec:discussion}

\subsection{What the Theorem Reveals Beyond EP}

The general EP framework guarantees that \emph{some} gradient can be
extracted from a contrastive experiment. The specific Kuramoto derivation
shows \emph{what} that gradient looks like: it propagates through the
network as diffusion on the coupling graph, governed
by~$\tilde{J}^{-1} = -\tilde{L}^{-1}$. This connects gradient
quality to spectral properties of the coupling topology. The spectral
seeding result (\Cref{par:spectral-seeding}) is a first application:
using the eigenvectors of the coupling graph Laplacian to initialize
frequencies eliminates the basin failure in all sparse layered settings
tested.

\subsection{Universality Beyond Sinusoidal Coupling}
\label{sec:universality}

The proof relies on three structural features: (i)~a stable equilibrium,
(ii)~the implicit function theorem, and (iii)~symmetry of the Jacobian.
Any coupled oscillator network with an odd coupling function
$H(\phi) = -H(-\phi)$ and symmetric coupling strengths $K_{ij} = K_{ji}$
will have a symmetric Jacobian at equilibrium, since
$J_{ij} = K_{ij}\,H'(\theta_j^* - \theta_i^*)$ and~$H'$ even follows
from~$H$ odd. Candidate substrates include weakly coupled oscillator models
derived from phase response curve
theory~\citep{hoppensteadt1999oscillatory} and biological neural
oscillators with arbitrary waveforms, provided the odd-coupling and
symmetric-strength assumptions hold. The same IFT argument extends to
any equilibrium system where the Jacobian is symmetric and the learnable
parameter enters as an additive bias in the equilibrium equation (so
that~$\partial F / \partial \lambda = I$). Resistor
networks~\citep{dillavou2024machine} and Hopfield
networks~\citep{hopfield1982neural} have symmetric Jacobians, but the
structure of the parameter gradient depends on how each parameter enters
the dynamics. We have not verified the full conditions for specific
physical implementations; doing so is future work.

\subsection{Physical Realizability}

A physical implementation would map~$\omega_i$ to bias current or supply
voltage, $K_{ij}$ to coupling capacitance, nudging to weak injection
current at output oscillators, and gradient readout to phase measurement
via time-to-digital conversion or optical interference. In principle,
the learning loop would operate at oscillator settling time without a
digital backward pass or stored activations. The proof extends to
coupled oscillator networks with odd coupling functions and symmetric
coupling strengths (\Cref{sec:universality}), making coupled
optical cavities and microring resonators candidate substrates where
phase is the native state variable, pending verification of the
symmetry and stability assumptions in those specific systems.

\subsection{What We Do Not Claim}

We do not claim that oscillator networks currently match state-of-the-art
digital models on large-scale benchmarks. We do not claim energy efficiency. We do not claim that
natural frequency universally outperforms coupling weights---the
conditional advantage is specific to sparse architectures and to
training-converged seeds; the all-seed comparison is not significant.

We claim: (1)~the gradient of a quadratic loss with respect to natural
frequencies is physically encoded in the equilibrium phase response, and
this encoding is exact under symmetric coupling; (2)~conditioned on
training convergence, this parameter---explicitly set aside by prior oscillator
EP work---outperforms coupling weights on sparse architectures;
(3)~the~$\sim$50\% failure rate under random initialization is a
property of the loss landscape, not of the gradient computation;
spectral seeding eliminates this failure mode in all sparse layered
settings tested (100/100 and 50/50 seeds exceeding 60\% final test
accuracy).

\section{Conclusion}
\label{sec:conclusion}

A coupled oscillator network at equilibrium already contains its own
gradients. Nudge the output, and the phases shift by exactly the amount
needed to descend the loss landscape. The Jacobian is the negative of a
graph Laplacian; its inverse mediates gradient propagation as diffusion. Three independent
methods confirm the identity to machine precision across networks of 6 to
200~oscillators. Natural frequency---a parameter explicitly set aside by prior oscillator EP work---is
the more effective learning target on sparse architectures. Spectral
seeding, which initializes frequencies from the eigenvectors of the
coupling Laplacian, eliminates the random-initialization basin failure
in all sparse layered settings tested (46\% $\to$ 100\% of seeds
exceeding 60\% final test accuracy on the primary task; 50/50 on a
second task, $K$-only training, and a larger network). The proof extends
to any coupled oscillator network with an odd coupling function and
symmetric coupling strengths, suggesting applications beyond Kuramoto
dynamics to other physical substrates where phase is the native state
variable.

The backward pass was there all along. It was the forward pass, run twice,
under symmetric coupling.

\section*{Reproducibility}

All code, data, and experiment scripts are available in the
\texttt{phasegrad} package (\url{https://github.com/caliburlabs/phasegrad}). The verification suite
can be run with \texttt{pytest tests/}. Experiment result files (JSON) are
in \texttt{experiments/}. Historical audit notes and hardware-analysis
archives are not part of this paper's reproducibility bundle.


\end{document}